\documentclass[twoside]{article}
\usepackage[accepted]{template/aistats2014_arxiv}

\usepackage{hyperref}
\usepackage{url}
\usepackage{graphicx,subfigure,graphics,enumitem,algorithm,algorithmic,amsmath,amssymb,amsfonts,amsthm,epstopdf} 

\makeatletter
  \def\title@font{\Large\bfseries}
  \let\ltx@maketitle\@maketitle
  \def\@maketitle{\bgroup%
    \let\ltx@title\@title%
    \def\@title{\resizebox{\textwidth}{!}{%
      \mbox{\title@font\ltx@title}%
    }}%
    \ltx@maketitle%
  \egroup}
\makeatother

\DeclareMathOperator*{\argmin}{arg\,min}

\newtheorem{theorem}{Theorem}[section]
\newtheorem{lemma}[theorem]{Lemma}
\theoremstyle{definition}

\theoremstyle{remark}

\newcommand{\norm}[1]{\lVert#1\rVert}
\newcommand{\Norm}[1]{\left\lVert#1\right\rVert}

\newcommand{\EE}[1]{\mathbb{E}\left[ #1 \right]}
\newcommand{\Idot}[2]{\left\langle #1 , #2 \right\rangle}
\newcommand{\idot}[2]{\langle #1 , #2 \rangle}

\newcommand{\iid}{\overset{iid}{\sim }}

\newcommand{\R}{\mathbb{R}}
\newcommand{\Z}{\mathbb{Z}}
\newcommand{\E}{\mathbb{E}}

\newcommand{\bP}{\mathbb{P}}

\newcommand{\bZ}{{\bf{Z}}}

\newcommand{\calF}{\mathcal{F}}

\newcommand{\calI}{\mathcal{I}}
\newcommand{\calX}{\mathcal{X}}

\newcommand{\calD}{\mathcal{D}}

\newcommand{\calN}{\mathcal{N}}
\newcommand{\calM}{\mathcal{M}}

\newcommand{\Ks}{K_\sigma}
\newcommand{\vY}{\vec{Y}}

\newcommand{\tp}{\tilde{p}}
\newcommand{\tP}{\tilde{P}}

\newcommand{\hP}{\widehat{P}}

\newcommand{\hpsi}{\hat{\psi}}

\newcommand{\bg}{\bar{g}}
\newcommand{\bA}{A}

\newcommand{\avec}{\vec{a}}

\newcommand{\ud}{\mathrm{d}}

\newcommand{\mxphi}{\varphi_{\mathrm{max}}}
\newcommand{\Ep}{\mathrm{E}}

\newcommand{\ie}{i.e.\ }

\newcommand{\where}{\mathrm{where }}

\newcommand{\Unif}{\text{Unif}}

\DeclareMathOperator*{\diam}{diam}
\DeclareMathOperator*{\Var}{Var}
\DeclareMathOperator*{\sign}{sign}
\DeclareMathOperator*{\GMM}{GMM}

\allowdisplaybreaks 
\begin{document}

\twocolumn[
\aistatstitle{Fast Distribution To Real Regression} 
\aistatsauthor{ \hspace{3mm}Junier B. Oliva \And Willie Neiswanger \And Barnab{\'a}s P{\'o}czos \And Jeff Schneider \And  Eric Xing}
\aistatsaddress{ Carnegie Mellon University } 
]

\begin{abstract}
We study the problem of distribution to real regression, where one aims to regress a mapping $f$ that takes in a distribution input covariate $P\in \calI$ (for a non-parametric family of distributions $\calI$) and outputs a real-valued response $Y=f(P) + \epsilon$. 
This setting was recently studied in \cite{poczos2013distribution}, where the ``Kernel-Kernel'' estimator was introduced and shown to have a polynomial rate of convergence. 
However, evaluating a new prediction with the Kernel-Kernel estimator scales as $\Omega(N)$. This causes the difficult situation where a large amount of data may be necessary for a low estimation risk, but the computation cost of estimation becomes infeasible when the data-set is too large. To this end, we propose the Double-Basis estimator, which looks to alleviate this big data problem in two ways: first, the Double-Basis estimator is shown to have a computation complexity that is independent of the number of of instances $N$ when evaluating new predictions after training; secondly, the Double-Basis estimator is shown to have a fast rate of convergence for a general class of mappings $f\in\calF$.
\end{abstract}

\section{Introduction}
A great deal of attention has been applied to studying new and better ways to perform learning tasks involving static finite vectors. Indeed, over the past century the fields of statistics and machine learning have amassed a vast understanding of various learning tasks like density estimation, clustering, classification, and regression using simple real valued vectors. However, we do not live in a world of simple objects. From the contact lists we keep, the sound waves we hear, 
and the distribution of cells we have, complex objects such as sets, functions, and distributions are all around us. Furthermore, with ever-increasing data collection capacities at our disposal, not only are we collecting more data, but richer and more bountiful complex data are becoming the norm. 

This paper aims to make learning on massive data-sets of distributions tractable; we study distribution to real regression (DRR) where input covariates are arbitrary distributions and output responses are real values. We provide an estimator that scales well with data-set size and is efficient at evaluation-time. Furthermore, we prove that the estimator achieves a fast rate of convergence for a broad class of functions.

We consider a mapping $f: \calI \mapsto \R$ that takes $P \in \calI$, an input distribution from a family of distributions $\calI$, and produces $Y$ a real-valued response as: 
\begin{align}
Y=f(P)+\epsilon, \where\ \EE{\epsilon}=0,\ \EE{\epsilon^2}\leq\sigma^2_\epsilon.
\end{align} 
Of course, it is infeasible to directly observe a distribution in practice. Thus, we will work on a data-set of $N$ input sample-sets/responses: 
\begin{gather}
\calD = \{(\calX_i,Y_i)\}_{i=1}^N,\ \where \label{eq:dataset}\\
\calX_i = \{ X_{i1},\ldots,X_{in_i} \},\ X_{ij}\iid P_i \in \calI,
\end{gather}
and $Y_i = f(P_i)+\epsilon_i$. Further, $P_i \iid \Phi$, where $\Phi$ is some measure over $\cal I$ (see Figure \ref{fig:gmodel}).

\begin{figure}[h!]
\label{fig:gmodel}
  \centering
    \includegraphics[width=0.4\textwidth]{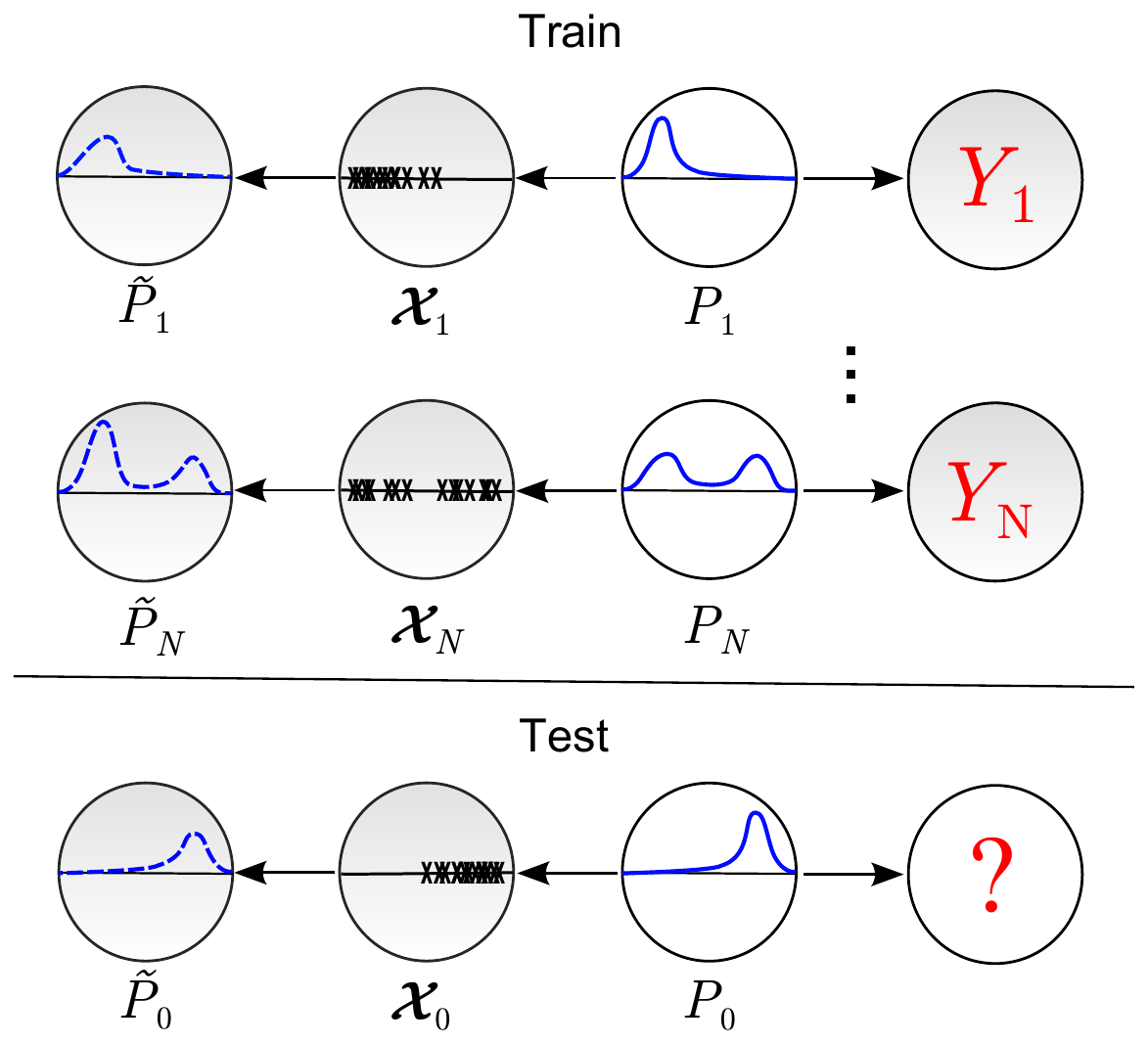}
    \caption{A graphical representation of our model. We observe a data-set of input sample-set/output response pairs $\{(\calX_i,Y_i) \}_{i=1}^N$, where $\calX_i = \{X_{i1},\ldots, X_{in_i} \}$, $X_{ij}\sim P_i$ and $Y_i = f(P_i) + \epsilon_i$, for some noise $\epsilon_i$. From these sample sets $\{\calX_i \}_{i=1}^N$ we build density estimates $\{\tP_i \}_{i=1}^N$ using projection series estimates \eqref{eq:coef-est}. These estimates will then be used in our response estimator \eqref{eq:OLSest}. }
    \vspace{-2.5mm}
\end{figure}

Many interesting problems across various domains fit the DRR model. For instance, one may be interested in studying the mapping that takes in the distribution of star locations in a galaxy and outputs the galaxy's age. Also, one may be consider a mapping that takes in the distribution of prices for stocks of a particular sector and outputs the future average change in stock price for that sector. 

In fact, many estimation tasks in statistics can be framed as a distribution to regression problem. For instance, in parameter estimation one studies a mapping that takes in a distribution (usually restricted to be in a parametric class of distributions) and outputs a corresponding parameter. We will see that our estimator can be used to leverage previously seen sample sets to outperform standard estimation procedures, to perform model selection when cross validation is expensive, or to perform parameter estimation when no analytical sample estimate is available. In effect, we shall show that this estimator, and the concept of distribution to real regression, is powerful enough to itself learn how to perform general statistical procedures. 

At its core, the problem of distribution to real value regression is a learning task over infinite dimensional objects (distributions) and  would benefit greatly from learning on data-sets with a large number of input/output pairs. Hence, this paper focuses on the case where one has a massive data-set in terms of instances, \ie $n_i = o(N)$. DRR for the case of general input distributions in a H\"{o}lder class and a smooth class of mappings has been previously studied in \cite{poczos2013distribution}. There, an estimator---the Kernel-Kernel estimator---analogous to the Nadaraya-Watson estimator \cite{tsybakov2008introduction} for functional distribution inputs was shown to have a polynomial rate of convergence. This rate is dependent on the dimensionality of the domains of the distributions, sample sizes, and a doubling dimension on the measure $\Phi$ over distributions, which, roughly speaking, controls the degrees of freedom of the input distributions. However, evaluating the estimator in \cite{poczos2013distribution} for new predictions scales as $\Omega(N)$ in the number of input/output instances in a data-set. Thus, the Kernel-Kernel estimator is not feasible for data-sets where the number of distributions, $N$, is in the high-thousands, millions, or even billions. Furthermore, the doubling dimension of $\Phi$ may be rather large, producing a slow convergence rate. In this paper we shall introduce an estimator for DRR, the Double-Basis estimator, 
which does not depend on $N$ for evaluating an estimate for a  new input distribution. Furthermore, we shall show that this estimator achieves a better rate of convergence that does not depend on the doubling dimension over a broad class of distribution to real mappings.

\section{Related Work}
As previously mentioned, the problem of DRR was studied in \cite{poczos2013distribution}, where the Kernel-Kernel estimator was introduced. Since the data-set one works with is \eqref{eq:dataset}, first one uses kernel density estimation (KDE) \cite{tsybakov2008introduction} on $\{\calX_1,\ldots,\calX_N\}$ to make density estimates $\{\tP_1,\ldots,\tP_N\}$. Similarly for an unseen query input sample set $\calX_0 \sim P_0$, one makes a KDE $\tP_0$. Then, the Kernel-Kernel estimator works as follows:
\begin{align}
&\hat{f}(\tP) = \sum_{i=1}^N W(\tP_i,\tP_0) Y_i,\ \where \label{eq:KKest}\\
&W(\tP_i,\tP_0) \nonumber \\
&= \begin{cases} \frac{K(D(\tP_i,\tP_0))}{\sum_j K(D(\tP_j,\tP_0))} &\mbox{if } \sum_j K(D(\tP_j,\tP_0))>0 \\
0 & \mbox{otherwise }. \end{cases} 
\end{align}
Here $K$ is taken to be a symmetric Kernel with bounded support, and $D$ is some metric over functions. Clearly, \eqref{eq:KKest} scales as $\Omega(N)$ in terms of the number of input distributions in ones data-set. Furthermore, if one uses a Gaussian KDE, and takes $D(\tP_i,\tP_0)=\norm{\tP_i-\tP_0}_2=\sqrt{\int(\tp_i-\tp_0)^2}$ (where $\tp_i$ is the pdf of $\tP_i$) and $n_i\asymp n$, then the computation required for evaluating \eqref{eq:KKest} is $\Omega(Nn^2)$.

DRR is related to the functional analysis, where one regresses a mapping whose input domain are functions \cite{ferraty2006nonparametric}. However, the objects DRR works over--distributions and their pdfs--are inferred through sets of samples drawn from the objects, with finite sizes. In functional analysis, the functions are inferred through observations of $(X,Y)$ pairs that are often taken to be an arbitrarily dense grid in the domain of the functions. For a comprehensive survey in functional analysis see \cite{ferraty2006nonparametric,ramsay2002applied}. Also, recently \cite{olivadistribution} studied the problem of distribution to distribution regression, where both input and output covariates are distributions. 

A common approach to performing ML tasks with distributions is to embed the distributions in a Hilbert space, then solve the tasks using kernel machines. Perhaps the most clear-cut of these methods is to fit a parametric model to distributions for estimating kernels \cite{jebara2004probability, jaakkola1999exploiting, moreno2003kullback}. Nonparametric methods over distributions have also been developed using kernels. For example, since we only observe distributions through finite sets, set kernels may be used \cite{smola2007hilbert}. Futhermore, the representer theorem was recently generalized for the space of distributions \cite{muandet2012learning}. Also, kernels based on nonparametric estimators of divergences have been explored \cite{poczos2012nonparametric,poczos2012image}. 

\section{Double-Basis Estimator}
We introduce the Double-Basis Estimator for DRR. First, we shall use orthonormal basis projection estimators \cite{tsybakov2008introduction} for estimating the densities of $P_i$ from $\calX_i$. Suppose that $\Lambda^l \subseteq \R^l$, the domain of input densities is compact s.t. $\Lambda = [a,b]$. Let $\{\varphi_i\}_{i\in\Z}$ be an orthonormal basis for $L_2(\Lambda)$. Then, the tensor product of $\{\varphi_i\}_{i\in\Z}$ serves as an orthonormal basis for $L_2(\Lambda^l)$; that is,
\begin{gather*}
\{\varphi_\alpha\}_{\alpha\in\Z^l} \quad \mathrm{where} \quad \varphi_\alpha(x) = \prod_{i=1}^l \varphi_{\alpha_i}(x_i),\ x\in \Lambda^l
\end{gather*}
serves as an orthonormal basis (so we have $\forall \alpha,\rho\in \Z^l,\ \langle\varphi_\alpha,\varphi_\rho \rangle= I_{\{\alpha=\rho\}}$).

Let $P\in \calI \subseteq L_2(\Lambda^l)$, then
\begin{gather}
p(x)=\sum_{\alpha\in\Z^l} a_\alpha(P)\varphi_\alpha(x)  \ \where\\
\quad a_\alpha(P) = \langle\varphi_\alpha, p\rangle = \int_{\Lambda^l} \varphi_\alpha(z)\ud P(z)\ \in \R \nonumber.
\end{gather}
where $p(x)$ denotes the probability density function of the distribution $P$.

Suppose that the projection coefficients $a(P) = \{a_\alpha(P)\}_{\alpha\in \Z^l}$ are as follows for $P\in\calI$:
\begin{gather}
\calI = \{ P : a(P) \in \Theta_l(\nu,\gamma,\bA),\ \norm{P}_2^2 \leq \bA \} \quad \where
\label{eq:sob-ellp}\\
\Theta_l(\nu,\gamma,\bA) = \left\{\{a_\alpha\}_{\alpha\in \Z^l} : \sum_{\alpha\in\Z^l} a_\alpha^2 \kappa_\alpha^2(\nu,\gamma) < \bA \right\},\nonumber\\
\kappa_\alpha^2(\nu,\gamma) = \sum_{i=1}^l(\nu_i|\alpha_i|)^{2\gamma_i}\ \mathrm{for}\ 
\nu_i,\gamma_i,\bA > 0. \nonumber
\end{gather}
See \cite{ingster2011estimation,laurent1996efficient} for other analyses with this type of assumption. The assumption in \eqref{eq:sob-ellp} will control the tail-behavior of projection coefficients and allow us to effectively estimate $P\in \calI$ using a finite number of projection coefficients on the empirical distribution of a sample.

Given a sample $\calX_i = \{X_{i1},\ldots,X_{in_i}\}$ where $X_{ij}\iid P_i \in \calI$, let $\hP_i$ be the empirical distribution of $\calX_i$; i.e. $\hP_i(X=X_{ij})=\frac{1}{n_i}$. Our estimator for $p_i$ will be:
\begin{align}
\tp_i(x) = \sum_{\alpha\ :\ \kappa_\alpha(\nu,\gamma)\leq t}a_\alpha(\hP_i)\varphi_\alpha(x) \quad \where \label{eq:coef-est}\\
a_\alpha(\hP_i) = \int_{\Lambda^l} \varphi_\alpha(z)\ud \hP_i(z) = \frac{1}{n_i}\sum_{j=1}^{n_i} \varphi_\alpha(X_{ij}) \label{eq:coef-hat}.
\end{align}
Choosing $t$ optimally\footnote{See appendix for details.} can be shown to lead to $\E[\norm{\tp_i-p_i}_2^2]=O(n_i^{-\frac{2}{2+\gamma^{-1}}})$, where $\gamma^{-1}=\sum_{j=1}^{l}\gamma_j^{-1}$, $n_i \rightarrow \infty$ \cite{nussbaum1983optimal}.

Next, we shall use random basis functions from Random Kitchen Sinks (RKS) \cite{rahimi2007random} to compute our estimate of the response. \cite{rahimi2007random} shows that if one has a shift-invariant kernel $K$ (in particular we consider the RBF kernel $K(x)=\exp(-x^2/2)$) then for $x,y \in \R^d$:
\begin{align}
&K(\Norm{x-y}_2/\sigma) \approx z(x)^Tz(y),\ \where\\
&z(x) \equiv\nonumber \\
& \sqrt{\tfrac{2}{D}}\left[\cos(\omega_1^Tx+b_1) \cdots \cos(\omega_D^Tx+b_D)\right]^T
\end{align}
with $\omega_i \stackrel{iid}{\sim} \calN(0,\sigma^{-2}I_d)$, $b_i \stackrel{iid}{\sim} \Unif[0,2\pi]$
Let $M_t = \{\alpha\ :\ \kappa_\alpha(\nu,\gamma) \leq t \} = \{\alpha_1,\ldots,\alpha_S\}$. First note that:
\begin{align*}
\idot{\tp_i}{\tp_j}  =&  \Idot{\sum_{\alpha\in M_t} a_{\alpha}(\hP_i) \varphi_{\alpha}}{\sum_{\alpha\in M_t} a_{\alpha}(\hP_j) \varphi_{\alpha}} \\
 =&  \sum_{\alpha\in M_t} \sum_{\beta\in M_t}a_{\alpha}(\hP_i) a_{\beta}(\hP_j)\Idot{\varphi_{\alpha}}{\varphi_{\beta}} \\
 =&  \sum_{\alpha\in M_t} a_{\alpha}(\hP_i) a_{\alpha}(\hP_j) \\
 =& \Idot{\avec_t(\hP_i)}{\avec_t(\hP_j)},
\end{align*}
where $\avec_t(\hP_i) = (a_{\alpha_1},\ldots,a_{\alpha_s})$, $M_t = \{\alpha_1,\ldots,\alpha_s\}$, and the last inner product is the vector dot product. Thus,
\begin{align*}
\Norm{\tp_i-\tp_j}_2 = \Norm{\avec_t(\hP_i)-\avec_t(\hP_j)}_2,
\end{align*}
where the norm on the LHS is the $L_2$ norm and the $\ell_2$ on the RHS. 

Consider a fixed $\sigma$. Let $\omega_i \stackrel{iid}{\sim} \calN(0,\sigma^{-2}I_s)$, $b_i \stackrel{iid}{\sim} \Unif[0,2\pi]$, be fixed. Let $K_\sigma(x) = K(x/\sigma)$. Then,
\begin{align}
\sum_{i=1}^N\theta_i K_\sigma(\norm{\tp_i-\tp_0}_2) \approx& \sum_{i=1}^N \theta_iz(\avec_t(\hP_i))^Tz(\avec_t(\hP_0)) \nonumber \\
=& \left(\sum_{i=1}^N \theta_iz(\avec_t(\hP_i))\right)^Tz(\avec_t(\hP_0)) \nonumber\\
=& \psi^T z(\avec_t(\hP_0)) \label{eq:lin_est_approx}
\end{align}
where $\psi = \sum_{i=1}^N \theta_iz(\avec_t(\hP_i)) \in \R^s$. Hence, we consider estimators of the form
\eqref{eq:lin_est_approx}; that is, we consider linear estimators in the non-linear space induced by $z(\avec_t(\cdot))$. In particular, we consider the OLS estimator using the data-set $\{(z(\avec_t(\hP_i)),Y_i)\}_{i=1}^N$ :
\begin{align}
\hat{f}(\tP_0) \equiv& \hpsi^T z(\avec_t(\hP_0))\ \where \label{eq:OLSest} \\
\hpsi \equiv& \argmin_\beta \norm{\vY-\bZ\beta}_2^2 \\
=& (\bZ^T\bZ)^{-1}\bZ^T\vY
\end{align}
for $\vY=(Y_1,\ldots,Y_N)^T$, and with $\bZ$ being the $N\times D$ matrix: $\bZ=[z(\avec_t(\hP_1))\cdots z(\avec_t(\hP_N)) ]^T$. 

\subsection{Evaluation Computational Complexity}
We see that after computing $\hpsi$, evaluating our estimator on a new distribution $P_0$ amounts to taking an inner product with a $D \times 1$ vector. Including the time required for computing $z(\avec_t(\hP_0))$, the computation required for the evaluation, $\hat{f}(\tP_0) = \hat{\psi}^Tz(\avec_t(\hP_0))$, is: one, the time for evaluating the projection coefficients $\avec_t(\hP_1)$, $O(sn)$; two, the time to compute the RKS features $z(\cdot)$, $O(Ds)$; three, the time to compute the inner product, $\langle \hat{\psi}, \cdot \rangle$, $O(D)$. Hence, the total time is $O(D+Ds+sn)$. We'll see that $D=O(n\log(n))$ and $s=O(n)$ hence the total run-time for evaluating $\hat{f}(\tP_0)$ is $O(n^2\log(n))$. Since we are considering data-sets where the number of instances $N$ far outnumbers the number of points per sample set $n$, $O(n^2\log(n))$ is a substantial improvement over $O(Nn^2)$.

\subsection{Ridge Double-Basis Estimator}
We note that a straightforward extension to the Double-Basis estimator is to use a ridge regression estimate on features $z(\avec_t(\cdot))$ rather than a OLS estimate. That is, for $\lambda\geq 0$ let
\begin{align}
\hpsi^T_\lambda \equiv& \argmin_\beta \norm{\vY-\bZ\beta}_2^2 + \lambda \norm{\beta}_2 \label{eq:ridgeest}\\
=& (\bZ^T\bZ+\lambda I)^{-1}\bZ^T\vY.
\end{align}
Clearly the Ridge Double-Basis estimator is still evaluated via a dot product with $\hpsi^T_\lambda$, and our above complexity analysis holds. Furthermore, we note that the Double-Basis estimator is a special case of the Ridge Double-Basis estimator with $\lambda =0$.

\section{Theory}
\subsection{Assumptions}
We shall assume the following:
\begin{enumerate}[label=\textbf{A.\arabic*}]
\item{ \label{asmp:sob}
{\em Sobolev Input Distributions.} Suppose that \eqref{eq:sob-ellp} holds.
}
\item{ \label{asmp:sob}
{\em RKHS Mapping.} We shall assume that $f \in \calF(\sigma,B)$ for $f: \calI \mapsto \R$, where $\sigma, B \in \R$ and
\begin{align}
\calF(\sigma,B) = \Big\{& f : f(P) = \sum_{i=1}^\infty \theta_i \Ks\left(G_i,P\right) ,  \\
&\where\ G_i \in \calI,\ \norm{\theta}_1 \leq B\Big\}.
\end{align}
Here we take $\Ks(G_i,P) = \Ks(\norm{g_i-p}_2) = K({\norm{g_i-p}_2}/{\sigma})$ to be a shift-invariant kernel. In particular, we take $K$ to be the RBF kernel: $K(x) = \exp(-x^2/2)$. Note further that:
\begin{align}
|K(x)-K(x')| \leq e^{-\frac{1}{2}}|x-x'|. \label{eq:lips}
\end{align}
}
\item{ \label{asmp:sampsize}
{\em Input Sample Set Sizes. } Suppose that $\forall i$ $|\calX_i| \asymp n$. 
}
\end{enumerate}

\subsection{Convergence Rate}
Since by \ref{asmp:sob} we have that $|f(P)|\leq B$, we consider an upperbound for the risk of a truncated version of our estimator \eqref{eq:OLSest}. Let $T_B(x) \equiv \sign(x)\min(|x|,B)$. For readability, let $Z(P)=z(\avec_t(P))$. Let a small real $\delta>0$ be fixed. We look to show that:
\begin{theorem}
\begin{align*}
&\EE{\left(T_B\left(\hat{\psi}^TZ(\hP_0)\right)-f(P_0)\right)^2} \\
&= O\left(n^{-1/(2+\gamma^{-1})}\right) + O\left(\frac{n\log(n)\log(N)}{N}\right)
\end{align*}
with probability at least $1-\delta$.
\end{theorem}
Roughly speaking, our proof will work as follows: first, we show that a population optimal linear model in the non-linear features $Z(\cdot)$ is close to the function $f$; then we will show that a population optimal linear model is close to the OLS (sample optimal) linear model.

Thus, we proceed to show that predictions from the optimal linear model using $Z(P_0)$ is close to $f(P_0)$, that is:
\begin{align*}
\frac{1}{2}\E_{P_0}\left[ \left(f(P_0) -\beta^TZ(\hP_0) \right)^2\right]
\end{align*}
is small, where $\beta$ is an optimal weight vector.
Note that $\beta$ minimizes:
\begin{align}
&\E\left[ \left(Y_0 -\beta^TZ(\hP_0) \right)^2\right] = \label{eq:opti_linear}\\
& \E\left[ Y_0^2\right] - 2\E\left[Y_0Z(\hP_0)\right]^T\beta +\beta^T\E\left[Z(\hP_0)Z(\hP_0)^T\right]\beta \nonumber.
\end{align}
Let
\begin{align}
\varsigma_i \equiv \sum_{j=1}^\infty \theta_j \left( \Ks\left(g_j,p_i\right) - Z(G_j)^T Z(\hP_i)\right).
\end{align}
Furthermore, note that:
\begin{align}
Y_i = f(p_i) + \epsilon_i = \sum_{j=1}^\infty \theta_j \Ks\left(g_j,p_i\right) + \epsilon_i. \label{eq:realmodel}
\end{align}
Let
\begin{align*}
\bg_i = \sum_{\alpha\in M_t }a_\alpha(G_i)\varphi_\alpha(x).
\end{align*}
Also, let $\avec_t(G_j) = (a_{\alpha_1}(G_j),\ldots,a_{\alpha_S}(G_j))^T$. When using kitchen sinks, we will see that $Y$ is approximately a linear model. Precisely, 
\begin{align*}
Y_i =& \sum_{j=1}^\infty \theta_j Z(G_j)^T Z(\hP_i) + \varsigma_i + \epsilon_i \\
=& \psi^T Z(\hP_i) + \varsigma_i + \epsilon_i,
\end{align*}
where $\psi = \sum_{i=1}^\infty \theta_i Z(G_i)$.
First we prove the following bound for the error using the optimal linear model $\beta$:
\begin{lemma}
\begin{align*}
\E_{P_0}\left[ \left(f(P_0) -\beta^TZ(\hP_0) \right)^2\right] \leq \E_{P_0}\left[ \varsigma_0^2\right]+4B\sqrt{\E_{P_0}\left[ \varsigma_0^2\right]}
\end{align*}
\end{lemma}
\begin{proof}

Since \eqref{eq:opti_linear} is a quadratic function bounded below, an optimal $\beta$ may be found by satisfying stationarity (taking the gradient \eqref{eq:opti_linear} and setting to zero). We take $\beta = \Sigma^{+}\Sigma_Y$ where $\Sigma = \E[Z(\hP_0)Z(\hP_0)^T]$ is the uncentered covariance matrix, $\Sigma^{+}$ is its Moore-Penrose inverse, and $\Sigma_Y=\E[Y_0Z(\hP_0)]$ is the vector of uncentered covariances to the response\footnote{Note that if $\Sigma$ is nonsingular, $\Sigma^{+}=\Sigma^{-1}$ and $\beta$ is unique.}. 
Hence,
\begin{align*}
&\frac{1}{2}\E_{P_0}\left[ \left(f(P_0) -\beta^TZ(\hP_0) \right)^2\right] \\
&=  \frac{1}{2}\E_{P_0}\left[ \left(f(P_0) -\Sigma_Y^T\Sigma^{+}Z(\hP_0) \right)^2\right]\\
&= \frac{1}{2}\E_{P_0}\left[ \left(f(P_0) \right)^2\right] - \Sigma_Y^T\Sigma^{+}\E_{P_0}\left[ f(P_0) Z(\hP_0) \right] \\
&\quad +  \frac{1}{2}\Sigma_Y^T\Sigma^{+}\E_{P_0}\left[ Z(\hP_0)Z(\hP_0)^T\right]\Sigma^{+}\Sigma_Y \\
&= \frac{1}{2}\E_{P_0}\left[ \left(\psi^T Z(\hP_0)+\varsigma_0 \right)^2\right] \\
&\quad - \Sigma_Y^T\Sigma^{+}\E_{P_0,\epsilon_0}\left[ (f(P_0)+\epsilon_0) Z(\hP_0) \right] \\
&\quad + \frac{1}{2}\Sigma_Y^T\Sigma^{+}\Sigma\Sigma^{+}\Sigma_Y\\
&= \frac{1}{2}\psi^T\Sigma\psi + \E_{P_0}\left[ \varsigma_0z(\hP_0)^T \right]\psi+ \frac{1}{2}\E_{P_0}\left[ \varsigma_0^2\right] \\
&\quad - \frac{1}{2}\Sigma_Y^T\Sigma^{+}\Sigma_Y.
\end{align*}
Also,
\begin{align*}
\Sigma_Y &= \E_{P_0,\epsilon_0}\left[(\psi^T Z(\hP_0))Z(\hP_0)+\varsigma_0Z(\hP_0)+\epsilon_0Z(\hP_0) \right] \\
&= \Sigma\psi + \E_{P_0}\left[\varsigma_0Z(\hP_0)\right].
\end{align*}
Thus,
\begin{align*}
&\Sigma_Y^T\Sigma^{+}\Sigma_Y \\
= & (\psi^T\Sigma + \E_{P_0}\left[\varsigma_0Z(\hP_0)^T\right])\Sigma^{+}(\Sigma\psi + \E_{P_0}\left[\varsigma_0Z(P_0)\right])\\
= & \psi^T\Sigma\Sigma^{+}\Sigma\psi + \E_{P_0}\left[\varsigma_0Z(\hP_0)^T\right]\Sigma^{+}\Sigma\psi\\
&+\psi^T\Sigma\Sigma^{+}\E_{P_0}\left[\varsigma_0Z(\hP_0)\right]\\
&+\E_{P_0}\left[\varsigma_0Z(\hP_0)^T\right]\Sigma^{+}\E_{P_0}\left[\varsigma_0Z(\hP_0)\right]\\
= & \psi^T\Sigma\psi + 2\E_{P_0}\left[\varsigma_0Z(\hP_0)^T\right]\Sigma^{+}\Sigma\psi\\
&+\E_{P_0}\left[\varsigma_0Z(\hP_0)^T\right]\Sigma^{+}\E_{P_0}\left[\varsigma_0Z(\hP_0)\right].
\end{align*}
Hence,
\begin{align}
&\frac{1}{2}\E_{P_0}\left[ \left(f(P_0) -\beta^TZ(\hP_0) \right)^2\right] \nonumber\\
&= \frac{1}{2}\psi^T\Sigma\psi + \E_{P_0}\left[ \varsigma_0Z(\hP_0)^T \right]\psi+ \frac{1}{2}\E_{P_0}\left[ \varsigma_0^2\right]\nonumber\\
&\quad -\frac{1}{2}\psi^T\Sigma\psi - \E_{P_0}\left[\varsigma_0Z(\hP_0)^T\right]\Sigma^{+}\Sigma\psi\nonumber\\
&\quad-\frac{1}{2}\E_{P_0}\left[\varsigma_0Z(\hP_0)^T\right]\Sigma^{+}\E_{P_0}\left[\varsigma_0Z(\hP_0)\right]\nonumber\\
&\leq \frac{1}{2}\E_{P_0}\left[ \varsigma_0^2\right]+4B\sqrt{\E_{P_0}\left[ \varsigma_0^2\right]}\label{eq:blin-bnd},
\end{align}
see Appendix for details on the last bound.
\end{proof}

\begin{lemma}
\begin{align*}
\E_{P_0}\left[ \varsigma_0^2\right] = O\left(n^{\frac{-2}{2+\gamma^{-1}}}\right)
\end{align*}
with probability at least $1-\delta$.
\end{lemma}
\begin{proof}
$|\varsigma_i| \leq \sum_{j=1}^\infty |\theta_j| \left| \Ks\left(g_j,p_i\right) - Z(G_j)^T z(\hP_i)\right|$
and
\begin{align*}
&\left| \Ks\left(g_j,p_i\right) - Z(G_j)^T Z(\hP_i)\right| \\
&\leq \left| \Ks\left(g_j,p_i\right) - \Ks\left(\bg_j,\tp_i\right) \right| \\
&+ \left| \Ks\left(\bg_j,\tp_i\right) - Z(G_j)^T Z(\hP_i)\right|.
\end{align*}
Also, using \eqref{eq:lips}:
\begin{align*}
&\left| \Ks\left(g_j,p_i\right) - \Ks\left(\bg_j,\tp_i\right) \right| \\
&\leq \frac{e^{-\frac{1}{2}}}{\sigma} \left|\norm{g_j - p_i}_2 - \norm{\bg_j - \tp_i}_2 \right|.
\end{align*}
Moreover, using the triangle inequality:
\begin{align*}
\left| \norm{g_j - p_i}_2 - \norm{\bg_j - \tp_i}_2 \right| \leq & \norm{g_j - \bg_j}_2 + \norm{\tp_i - p_i}_2 . 
\end{align*}
Thus,
\begin{align*}
&\EE{\left| \Ks\left(g_j,p_i\right) - \Ks\left(\bg_j,\tp_i\right) \right|}\\
&\leq  \EE{\frac{e^{-\frac{1}{2}}}{\sigma} \left(\norm{g_j - \bg_j}_2 + \norm{\tp_i - p_i}_2 \right)} \\
&= O\left(n^{\frac{1}{2+\gamma^{-1}}}\right),
\end{align*}
where the last line follows\footnote{\label{note1}See Appendix for details.} by choosing $t\asymp n^{\frac{1}{2+\gamma^{-1}}}$, and the expectation is w.r.t. $\calX_i\sim P_i$, $P_i\sim\Phi$. 

Also, note that the dimensionality of $\avec_t(G_i)$ and $\avec_t(\hP_i)$ is\footnotemark[\value{footnote}] $S=|M(t)|=O(n^{{\gamma^{-1}}/{(2+\gamma^{-1}})})$. Let $\calM = \{ v\in \R^S \ : \ \norm{v}_2^2\leq\bA\}$. Then, $\avec_t(G_j),\ \avec_t(\hP_i)\in \calM$. Hence, by {\em Claim 1} in \cite{rahimi2007random}: 
\begin{align*}
&\bP\left[ \sup_{u,v\in\calM}|K(u,v)-z(u)^Tz(v)|\geq \xi \right] \\
&\leq 2^8\left(\frac{\sqrt{S}\diam(\calM)}{\sigma \xi} \right)^2 \exp\left(- \frac{D \xi^2}{4(S+2)} \right).
\end{align*}
Thus, with probability at least $1-\delta$:
\begin{align*}
\sup_{u,v\in\calM}|K(u,v)-z(u)^Tz(v)|< n^{-\frac{1}{2+\gamma^{-1}}},
\end{align*}
if we choose $D$ such that:
\begin{align*}
D &= \Omega\left( 4(S+4)n^{\frac{2}{2+\gamma^{-1}}}\log\left( \delta^{-1} 2^{10} \bA S n^{\frac{2}{2+\gamma^{-1}}}/\sigma^2 \right) \right),
\end{align*}
which is satisfied setting $D \asymp n\log(n)$.

Hence, probability at least $1-\delta$:
\begin{align*}
&\frac{1}{2}\E_{P_0}\left[ \varsigma_0^2\right]\\
&\leq \E_{P_0}\Bigg[ \Bigg(\sum_{j=1}^\infty |\theta_j| \Big(\tfrac{e^{-\frac{1}{2}}}{\sigma} \norm{g_j - \bg_j}_2 + \tfrac{e^{-\frac{1}{2}}}{\sigma}\norm{\tp_0 - p_0}_2 \\
&\quad\quad\quad\quad +\left| \Ks\left(\bg_j,\tp_0\Big)- Z(G_j)^T Z(\hP_0)\right|  \right)  \Bigg)^2 \Bigg]\\
&= \E_{P_0}\left[ \left(\sum_{j=1}^\infty |\theta_j| \left(\norm{\tp_0 - p_0}_2 + O\left(n^{\frac{1}{2+\gamma^{-1}}}\right) \right)  \right)^2 \right] \\
&=  \left(\sum_{j=1}^\infty |\theta_j|  \right)^2\E_{P_0}\left[  \left(\norm{\tp_0 - p_0}_2 + O\left(n^{\frac{1}{2+\gamma^{-1}}}\right) \right)^2 \right] \\
&=  O\left(n^{\frac{-2}{2+\gamma^{-1}}}\right).
\end{align*}
\end{proof}
Thus, we see that $f(P)$ is close to the linear model in the non-linear spaced induced by the $O(n\log(n))$ features $Z(\cdot)$:

Then, {\em Theorem 11.3 of} \cite{gyorfi2002distribution} states that the estimated linear predictor $\hat{\beta}\in\R^d$ has an error to the mean conditional response (when truncated) relative an optimal linear predictor $\beta$ as follows:
\begin{align}
&\EE{\left(T_B\left(\hat{\beta}^Tx\right)-\EE{y|x}\right)^2} \leq \nonumber\\
&8 \EE{\left(\beta^Tx-\EE{y|x}\right)^2}+ O(\max\{\sigma_\epsilon^2,B^2\}d \log(N)/N) \label{eq:claim-rate}.
\end{align}
Using our notation we have that:
\begin{align}
&\EE{\left(T_B\left(\hat{\psi}^TZ(\hP_0)\right)-f(P_0)\right)^2} \nonumber\\
&= O\left(n^{-1/(2+\gamma^{-1})}\right) + O\left(\frac{n\log(n)\log(N)}{N}\right) \label{eq:rate},
\end{align}
where we have bounded $\EE{\left(T_B\left(\hat{\beta}^Tx\right)-\EE{y|x}\right)^2}$ using Lemmas 4.2 and 4.3, giving us our desired rate.

\section{Experiments}
We perform experiments that demonstrate the ability of the Double-Basis estimator to learn distribution-to-real mappings from large training datasets, which can be applied to yield fast, accurate, and useful predictions. We illustrate this on a few statistical estimation tasks, which aim to take a set of samples from a distribution as input and yield some estimated quantity as output. For many such tasks, we can generate large amounts of relevant output quantities and associated input samples synthetically, and can train the Double-Basis estimator on these big datasets, giving us an automated procedure to learn a mapping for these statistical estimation tasks. We will show that, in some cases, this mapping can be more accurate, faster, and more robust than existing statistical procedures.

In all of the following experiments, we train on data of the form $\calD = \{(\calX_i,Y_i)\}_{i=1}^N$.

\subsection{Synthetic Mapping}
First, we look to emphasize the computational improvement in evaluation time of the Double-Basis estimator over the Kernel-Kernel estimator using experiments with synthetic data. Our experiments are as follows. We first set $N\in\{1\Ep4,1\Ep5,1\Ep6\}$. Then, we generate a random mapping $f$ such that $f(P) = \sum_{i=1}^{10}\theta_i\Ks(G_i,P)$. We took $\sigma=1$,  $\theta_i \sim \Unif[-5,5]$, and $G_i$ to be the pdf of a mixture of two truncated Gaussians (each with weight $.5$) on the interval $[0,1]$, whose mean locations are chosen uniformly at random in $[0,1]$, and whose variance parameters are taken uniformly at random in $[.05,.1]$. For $j=\{1,\ldots,N\}$ we also set $P_j$ to be a randomly generated mixture of two truncated Gaussians as previously described. We then generate $Y_i$ under the the noiseless case, i.e $Y_i = f(P_i)$ (kernel values were computed numerically). Then, we generated $\calX_i=\{X_{i1},\ldots,X_{in}\}$ where $n\propto N^{3/5}$ and $X_{i1}\iid P_i$. $\tP_i$ was then estimated using the samples $\calX_i$.

We compared the performance of both the Double-Basis (BB), and the Kernel-Kernel (KK) estimator on a separate test set of $\calD_t = \{(\calX_j,Y_j)\}_{j=1}^{N_t}$ where $N_t = 1\Ep5$, that was generated as $\calD$ was. We measured performance in terms of mean squared error (MSE) and mean evaluation time per new query $\calX_0$ (Figures \ref{fig:big_mse} and \ref{fig:big_time} respectively). One can see that in this case both estimators have similar MSEs, with the BB estimator doing somewhat better in each configuration of the data-set size. 
However, one can observe a striking difference in the average time to evaluate a new estimate $\hat{f}(\tP)$. Figure \ref{fig:big_time} is presented in a log scale, and illustrates the Kernel-Kernel estimator's lack of scaling on data-set size, $N$. On the other hand, the Double-Basis estimator is considerbly efficient even at large $N$ and has a speed-up of about $\times12$, $\times67$, and $\times139$ over the Kernel-Kernel estimate for $N=1\Ep4,1\Ep5,1\Ep6$ respectively.

\begin{figure}
        \centering
        \subfigure[Estimation Error]{\raisebox{0mm}{\label{fig:big_mse}\includegraphics[width=.2325\textwidth]{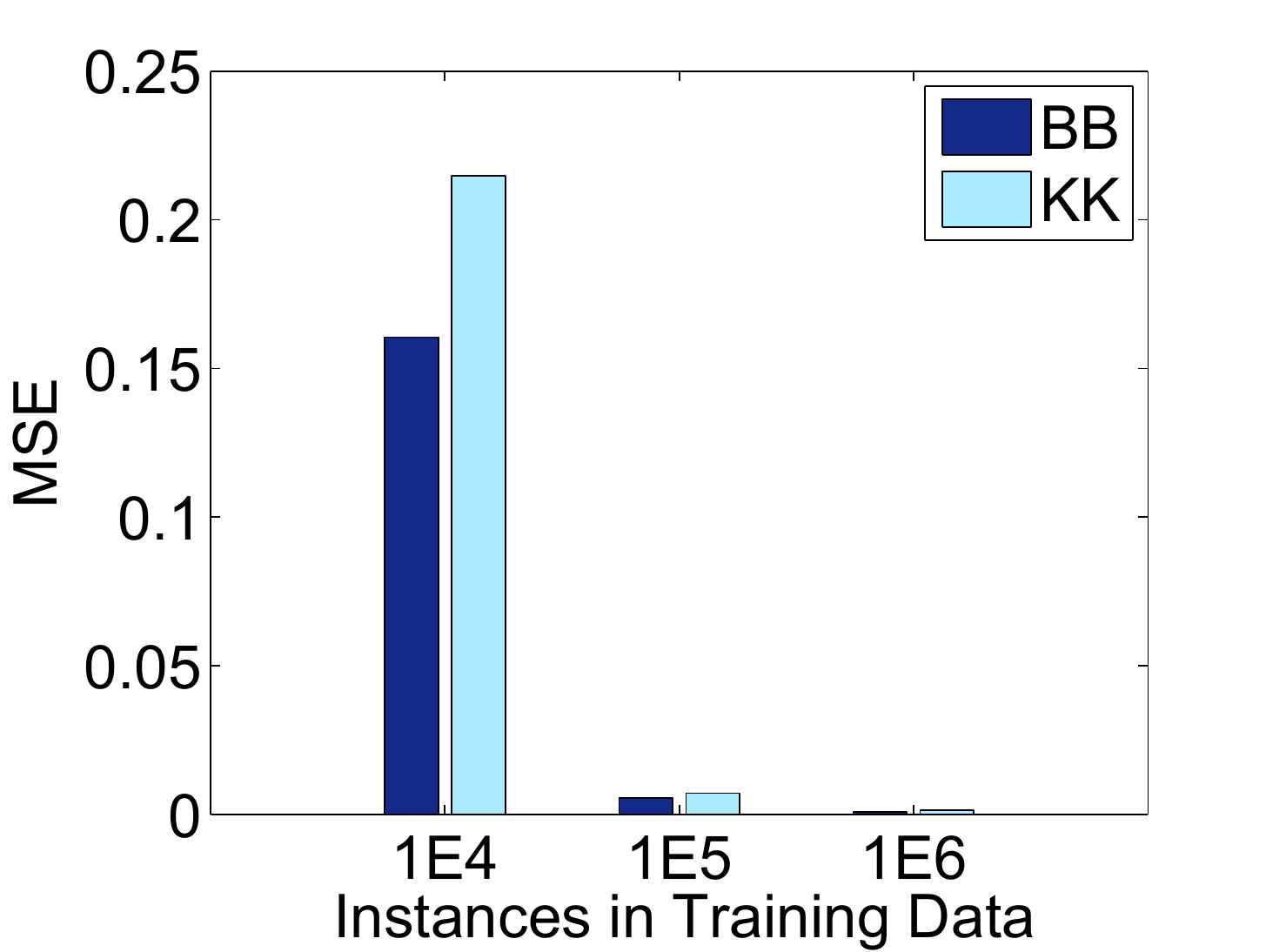}}}
        \subfigure[Estimation Time]{\raisebox{0mm}{\label{fig:big_time}\includegraphics[width=.2325\textwidth]{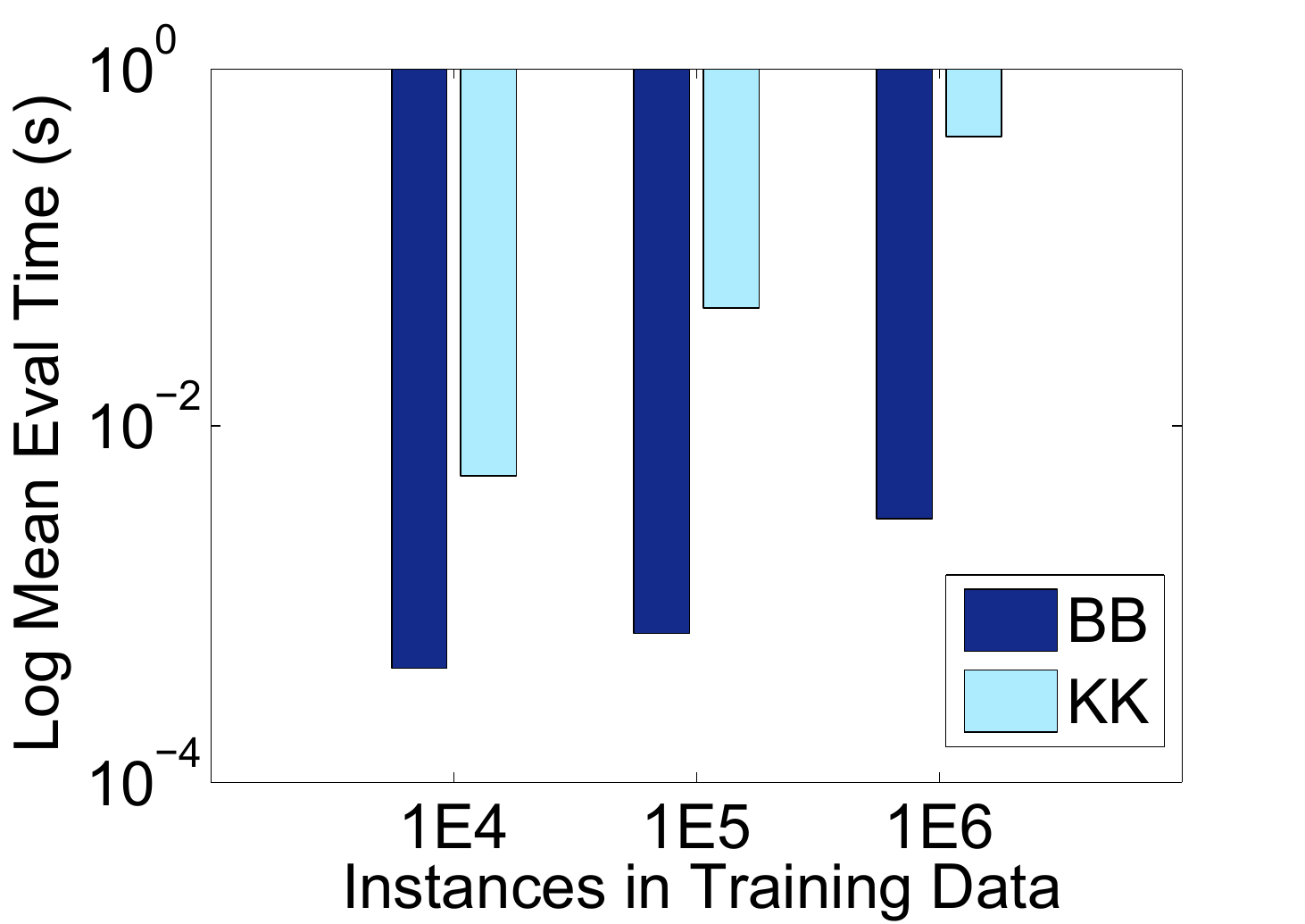}}}
        \vspace{-0.3cm}
        \caption{Results on predicting synthetic mapping $f$. }
\vspace{-0.4cm}
\end{figure}
\vspace{-0.2cm}
\subsection{Choosing $k$: model selection for Gaussian mixtures}
\vspace{-0.2cm}
Many common statistical tasks involve producing a mapping from a distribution to a real value, and may be tackled using DRR. One such task is that of model selection, where one is given a set $\calX_0 = \{X_{01},\ldots,X_{0n_0}\}$ drawn from an unknown distribution $P$ and wants to find some parameter that is indicative of the complexity of the true distribution. In other words, the mapping of interest takes in a distribution and outputs a hyperparameter of the distribution that is often illustrative of the distribution's complexity. 

In particular, we shall consider the model selection problem of selecting $k$, the number of components in a Gaussian mixture model (GMM). GMMs are often used in modeling data, however the selection of how many components to use is often a difficult choice. Attempting an MLE fit to training data will lead to choosing $k=n_0$ with one mixture component corresponding to each data-point. Hence, in order to effectively select $k$, one must fit a GMM for each potential choice of $k$ using an algorithm such as the expectation maximization algorithm (EM)~\cite{moon1996expectation}, then select the choice of $k$ that optimizes some score. In practice this often becomes computationally expensive. Typically scores used include Akaike information criterion (AIC), Bayesian information criterion (BIC), or a cross-validated data-fitting score on a holdout set (CV). We note that often GMMs are used to cluster data, where each data-point $X_{0i}$ is a assigned to a cluster based on which mixture component most likely generated it. Hence, the problem of selecting the number of mixture components in a GMM is closely related to the problem of selecting the number of clusters to use, which is itself a difficult problem.

Since selecting $k$ in GMMs is a DRR problem, and it is a relatively smooth mapping (that is, similar distributions should have a similar number of components), we hypothesize that one may learn to perform model selection in GMMs using the Double-Basis estimator. Particularly, by using a supervised dataset of $\{$sample-set, $k \}$ pairs, the Double-Basis estimator will be able to leverage previously seen data to perform model selection for a new unseen input sample set. 


Our experiment proceeds as follows. We can generate our own training data for this task by randomly drawing a value for $k$ (over some bounded range), then drawing 2-dimensional Gaussian mixture parameters for each of the $k$ components\footnote{See Appendix for figures of typical GMMs.}, and finally drawing samples from each Gaussian. That is, we generate $N=28,000$ input sample set/$k$ response pairs: $\calD = \{(\calX_i,k_i)\}_{i=1}^N$,  where $\calX_i = \{X_{i1},\ldots,X_{in}\}$, $X_{ij}\in\R^2$, $X_{ij}\iid\GMM(k_i)$, $k_i\sim\Unif\{1,\ldots,10\}$, and $\GMM(k_i)$ is a random GMM generated as follows, for $j=1,\ldots,k_i$: the prior weights for each component is taken to be $\pi_j=1/k_i$; the means are $\mu_j\sim\Unif[-5,5]^2$; and covariances are $\Sigma_j = a^2AA^T+B$, where $a\sim\Unif[1,2]$ $A_{uv}\sim\Unif[-1,1]$, and $B$ is a diagonal $2\times2$ matrix with $B_{uu}\sim\Unif[0,1]$. We train and get results using $n$ in the following range: $n \in {10,25,50,200,500,1000}$. We perform model selection using the mapping learned by the Ridge Double-Basis estimator \eqref{eq:ridgeest} (denoted BB in experiments), and compare it with model selection via AIC, BIC, and CV. We also compare agasint the Kernel-Kernel (KK) smoother. For all methods we computed the mean squared error between the true and predicted value for $k$ over 2000 test sample sets (Figure~\ref{fig:gmm_results}). We see that the Double-Basis estimator has both the lowest MSE and the lowest average evaluation time for computing a new prediction. In fact, the Double-Basis estimator can carry out the model selection prediction orders of magnitude faster than the CV, AIC, or BIC procedures.

\begin{figure}
        \centering
        \subfigure[Estimation Error]{\raisebox{0mm}{\includegraphics[width=.235\textwidth]{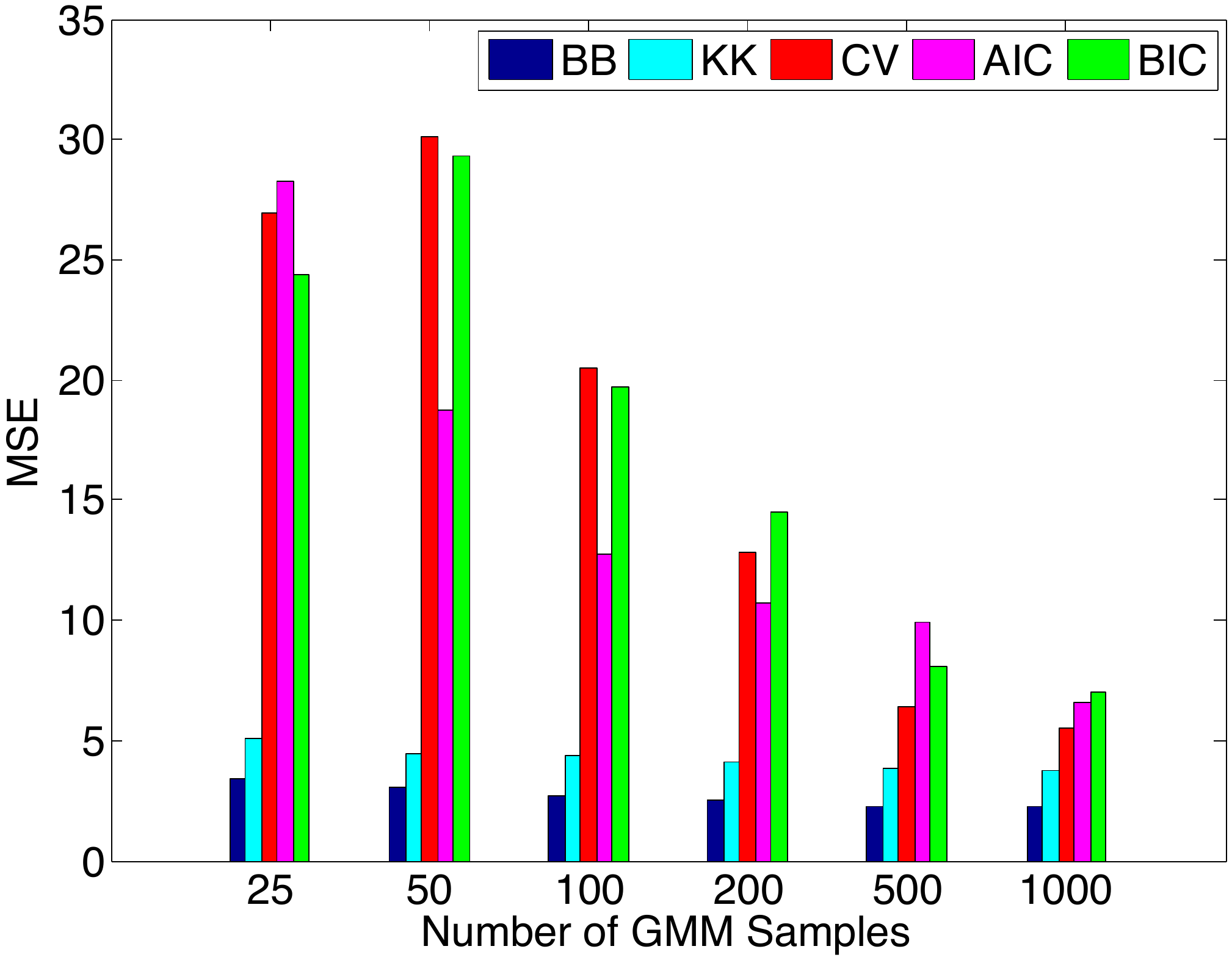}}}
        \subfigure[Estimation Time]{\raisebox{1mm}{\includegraphics[width=.23\textwidth]{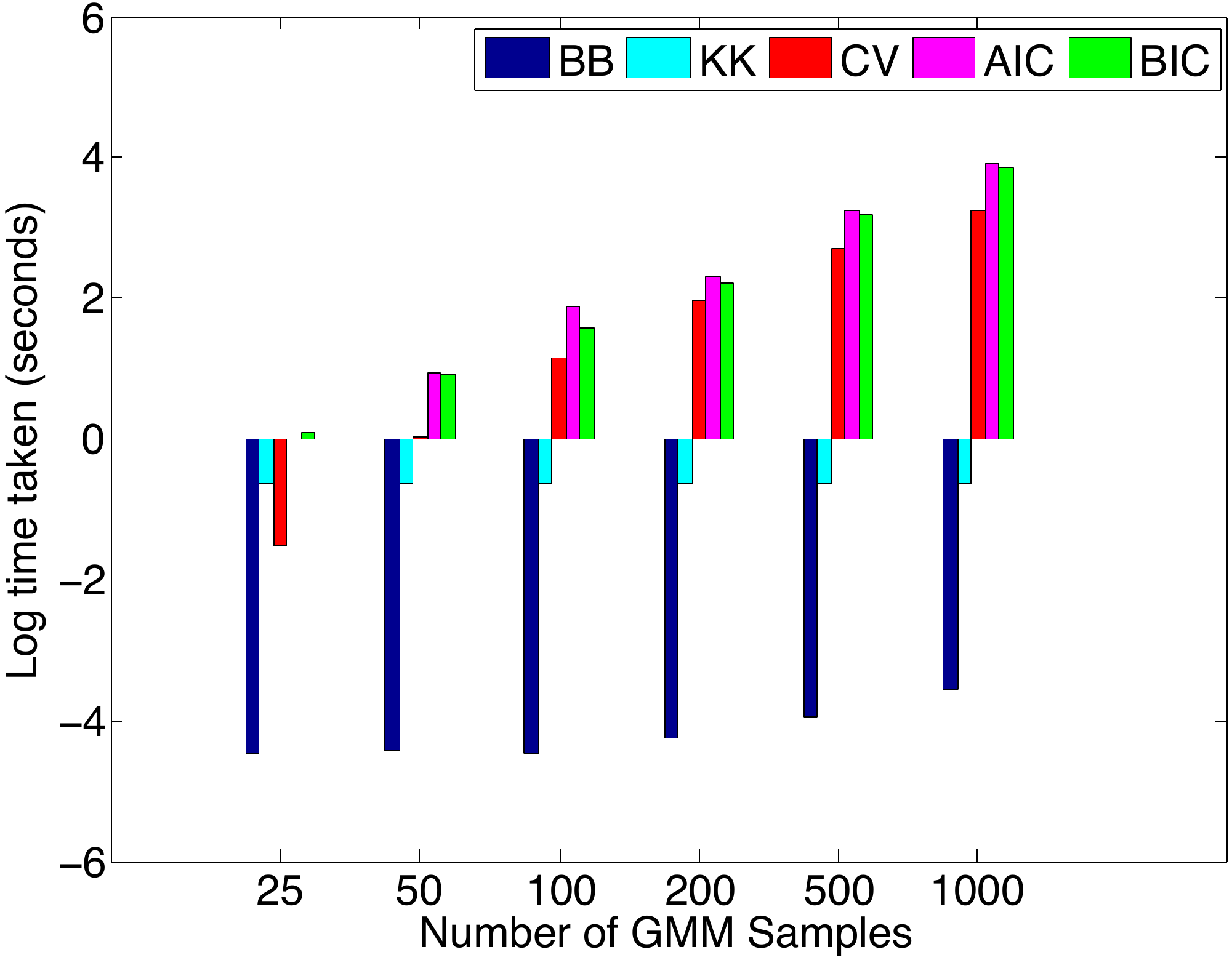}}}
        \vspace{-0.3cm}
        \caption{Results on predicting the number of GMM components. }
        \label{fig:gmm_results}
\vspace{-0.5cm}
\end{figure}


\vspace{-0.2cm}
\subsection{Low Sample Dirichlet Parameter Estimation}
\vspace{-0.2cm}
Similar to model selection, general parameter point estimation is a statistical task that may be posed as a DRR problem. That is, in parameter estimation one considers a set $\calX_0 = \{X_{01},\ldots,X_{0n_0}\}$ where points are drawn from some distribution $P(\eta_0)$ that is parameterized by $\eta_0$, and attempts to estimate $\eta_0$. In particular, we use DRR and the Double-Basis estimator to perform parameter estimation for Dirichlet distributions. The Dirichlet distribution is a family of continuous, multivariate distributions parameterized by a vector $\alpha \in \mathbb{R}_+^d$, with support over the $d$-simplex. Since every element of the support sums to one, the Dirichlet is often used to model distributions over proportion data. As before, we hypothesize that the Double-Basis estimator will serve as a way to leverage previously seen sample sets to help perform parameter estimation for new unseen sets. Effectively, our estimator will be able to ``boost'' the sample-size of a new input sample set by making use of what was learned on previously seen labeled sample sets.

Maximum likelihood parameter estimation for $\alpha$, given a set of Dirichlet
samples, is often performed via iterative optimization algorithms, such as 
gradient ascent or Newton's method \cite{minka2000estimating}, as a 
closed form solution for the 
MLE does not appear to exist in the literature. In this experiment, we aim to use 
DDR as a new method for Dirichlet parameter estimation. In particular, we generate samples from 
Dirichlet distributions with parameter values in a prespecified range, and use these as training 
data to learn a mapping from data samples to Dirichlet $\alpha$ parameter values.

In our experiments, we first fix
the range of $\alpha$ values to be constrained such that the $i^{th}$ component $\alpha_i \in [0.1,10]$. 
For each $28,000$ training instances, we uniformly sample a new $\alpha$ parameter vector within this range, 
and then generate $n$ points from the associated Dirichlet$(\alpha)$ distribution, where
$n \in \{10,25,50,200,500,1000\}$. We compare the Ridge Double-Basis estimator \eqref{eq:ridgeest} against a Newtons-method procedure for maximum likelihood 
estimation (MLE) from the fastfit toolbox \cite{minka2006fastfit}, and again against the Kernel-Kernel smoother. 
For all methods, for each $n$, we compute the mean squared error between the true and the estimated $\alpha$ parameter. 
We also record the time taken to perform the parameter estimation in each case. Results are shown in 
Figure~\ref{fig:dir_results}. We see that the Double-Basis estimator achieves the lowest MSE in 
all cases, and has the lowest average compute time. It is worth noting that the Double-Basis estimator 
performs particularly well relative to the MLE in cases where the sample size is low. We envision that 
Double-Basis estimator is particularly well suited for cases where one hopes to quickly, and in an 
automatic fashion, construct an estimator that can achieve highly accurate results for a 
statistical estimation problem for which an
optimal estimator might be hard to derive analytically.


\begin{figure}
        \centering
        \subfigure[Estimation Error]{\raisebox{0mm}{\includegraphics[width=.23\textwidth]{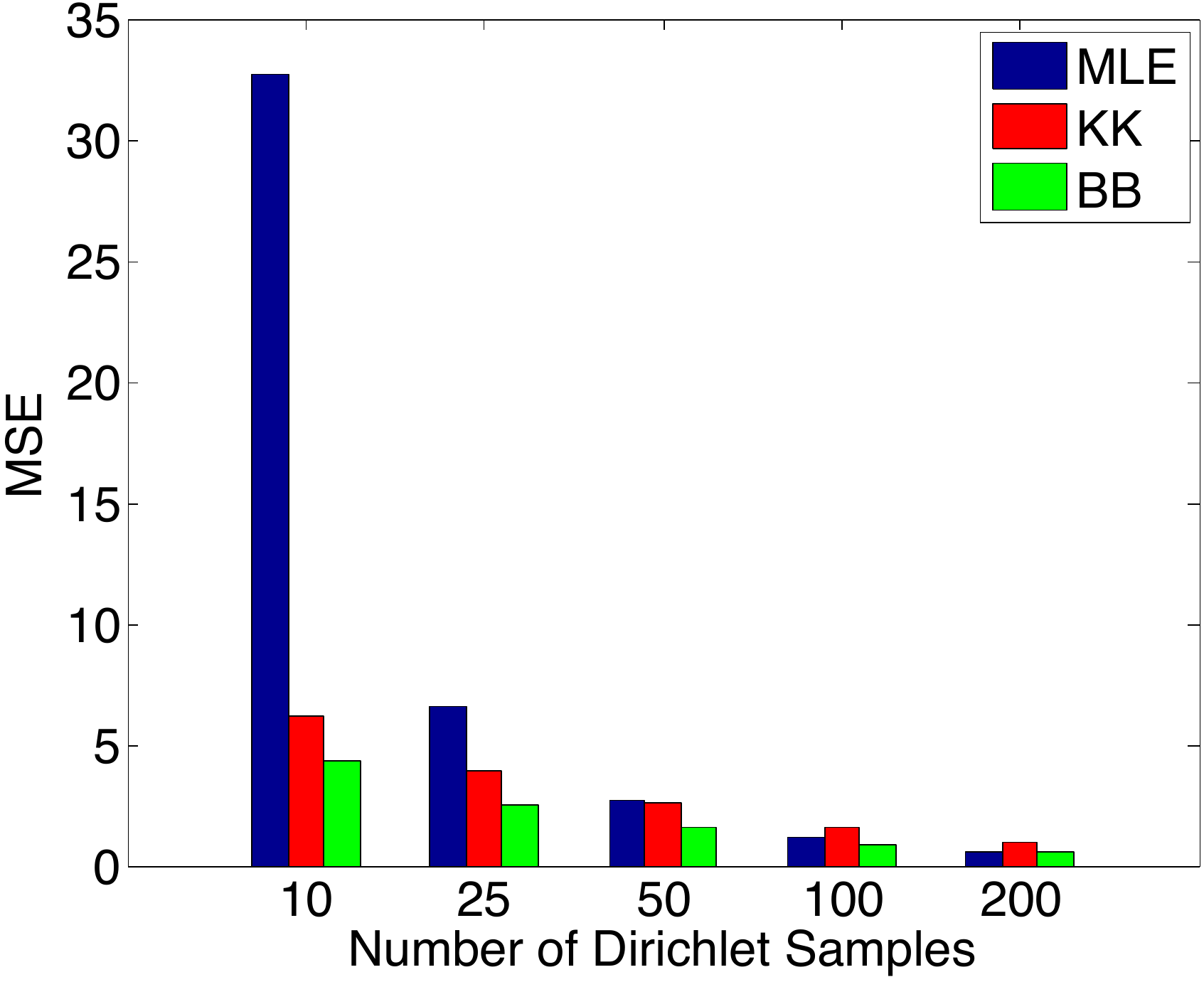}}}
        \subfigure[Estimation Time]{\includegraphics[width=.235\textwidth]{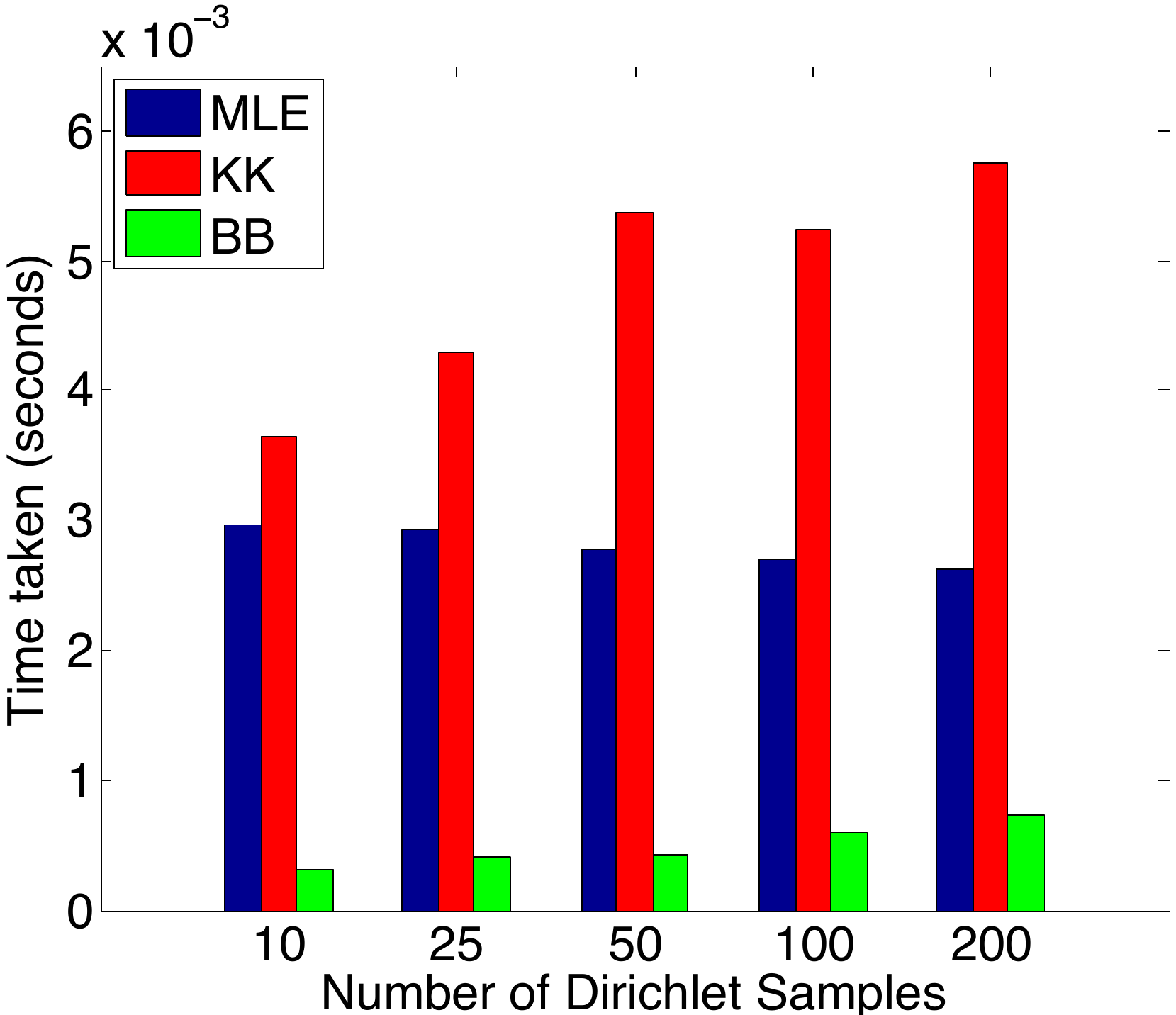}}
        \vspace{-0.3cm}
        \caption{Results predicting Dirichlet parameters. }
        \label{fig:dir_results}
\vspace{-0.5cm}
\end{figure}

\vspace{-0.2cm}
\section{Conclusion}
\vspace{-.2cm}
In conclusion, this paper presents a new estimator, the Double-Basis (BB) estimator, for performing distribution to real regression. In particular, this estimator scales independently of $N$ (the number input sample-set/response pairs) in a large dataset for performing evaluations for response predictions. This is a great improvement over the linear scaling with $N$ that the Kernel-Kernel (KK) estimator has and allows one to explore DRR in new domains with large collections of distributions, such as astronomy and finance. Furthermore, we prove an efficient upper bound on the risk for the BB estimator. Also, we empirically showed the improved scaling of the Double-Basis estimator, as well improvements in risk over the KK estimator. It is worth noting that while the BB estimator regresses a mapping in a nonlinear space (induced by RKS features), the KK estimator is outputs only a weighted average of training set responses.
\vspace{-.2cm}
\subsubsection*{Acknowledgements}
\vspace{-.2cm}
This work is supported in part by NSF grants IIS1247658 and IIS1250350.

\clearpage
{\small \bibliographystyle{amsplain}
        \bibliography{main}}
        
\clearpage
\subsection*{Appendix}

\subsubsection*{Details on Bound \eqref{eq:blin-bnd}}
\begin{align*}
&\frac{1}{2}\E_{P_0}\left[ \left(f(P_0) -\beta^TZ(\hP_0) \right)^2\right] \nonumber\\
&= \frac{1}{2}\psi^T\Sigma\psi + \E_{P_0}\left[ \varsigma_0Z(\hP_0)^T \right]\psi+ \frac{1}{2}\E_{P_0}\left[ \varsigma_0^2\right]\nonumber\\
&\quad -\frac{1}{2}\psi^T\Sigma\psi - \E_{P_0}\left[\varsigma_0Z(\hP_0)^T\right]\Sigma^{+}\Sigma\psi\nonumber\\
&\quad-\frac{1}{2}\E_{P_0}\left[\varsigma_0Z(\hP_0)^T\right]\Sigma^{+}\E_{P_0}\left[\varsigma_0Z(\hP_0)\right]\nonumber\\
&\leq \frac{1}{2}\E_{P_0}\left[ \varsigma_0^2\right]+ \E_{P_0}\left[\varsigma_0Z(\hP_0)^T\left(\psi-\Sigma^{+}\Sigma\psi\right)\right],
\end{align*}
since $\Sigma^{+}$ is PSD. Let $\Sigma=USU^{-1}$ be the eigen-decomposition of $\Sigma$; i.e. $S$ is diagonal matrix of decreasing eigenvalues and $U$ is a real unitary matrix and $U^{-1}=U^T$. Then, $\Sigma^{+}=US^+U^{-1}$, where $S^+$ is the diagonal matrix where $(S^+)_{ii}=1/(S)_{ii}$ if $(S)_{ii}\neq0$ and $(S^+)_{ii}=0$ if $(S)_{ii}=0$. Furthermore, let $r=\mathrm{rank}(\Sigma)$, and $I_r$ be the diagonal matrix with $(I_r)_{ii}=1$ for $i\leq r$ and $(I_r)_{ii}=0$ for $i> r$. Hence:
\begin{align*}
\norm{\Sigma^{+}\Sigma\psi}_2^2 &= \psi^T\Sigma\Sigma^{+}\Sigma^{+}\Sigma\psi \\
&= \psi^TUSU^{-1}US^+U^{-1}US^+U^{-1}USU^{-1}\psi \\
&= \psi^TUI_rU^{-1}\psi \\
&\leq \psi^TUIU^{-1}\psi \\
&\leq \norm{\psi}_2^2.
\end{align*}
Furthermore,
\begin{align*}
\norm{\psi}_2 \leq \sum_{i=1}^\infty |\theta_i| \norm{Z(G_i)}_2 \leq \sqrt{2} B.
\end{align*}
Hence,
\begin{align*}
&\frac{1}{2}\E_{P_0}\left[ \left(f(P_0) -\beta^TZ(\hP_0) \right)^2\right] \nonumber\\
&\leq \frac{1}{2}\E_{P_0}\left[ \varsigma_0^2\right]+ \E_{P_0}\left[\varsigma_0Z(\hP_0)^T\left(\psi-\Sigma^{+}\Sigma\psi\right)\right]\\
&\leq \frac{1}{2}\E_{P_0}\left[ \varsigma_0^2\right]+ \E_{P_0}\left[|\varsigma_0||Z(\hP_0)^T\left(\psi-\Sigma^{+}\Sigma\psi\right)|\right]\\
&\leq \frac{1}{2}\E_{P_0}\left[ \varsigma_0^2\right]+ \E_{P_0}\left[|\varsigma_0|\norm{Z(\hP_0)}_2\norm{\psi-\Sigma^{+}\Sigma\psi}_2\right]\\
&\leq \frac{1}{2}\E_{P_0}\left[ \varsigma_0^2\right]+ \sqrt{2}(\norm{\psi}_2+\norm{\Sigma^{+}\Sigma\psi}_2)\E_{P_0}\left[|\varsigma_0|\right]\\
&\leq \frac{1}{2}\E_{P_0}\left[ \varsigma_0^2\right]+ \sqrt{2}(\norm{\psi}_2+\norm{\psi}_2)\E_{P_0}\left[|\varsigma_0|\right]\\
&\leq \frac{1}{2}\E_{P_0}\left[ \varsigma_0^2\right]+ \sqrt{2}(2\sqrt{2} B)\E_{P_0}\left[|\varsigma_0|\right]\\
&\leq \frac{1}{2}\E_{P_0}\left[ \varsigma_0^2\right]+ 4B\sqrt{\E_{P_0}\left[\varsigma_0^2\right]},
\end{align*}
where the last line follows from Jensen's inequality.
\subsubsection*{GMM Figures}
See Figure \ref{fig:typ-GMM}.
\begin{figure}[h!]
        \centering
        \subfigure[$k=4$ True Density]{\raisebox{0mm}{\includegraphics[width=.235\textwidth]{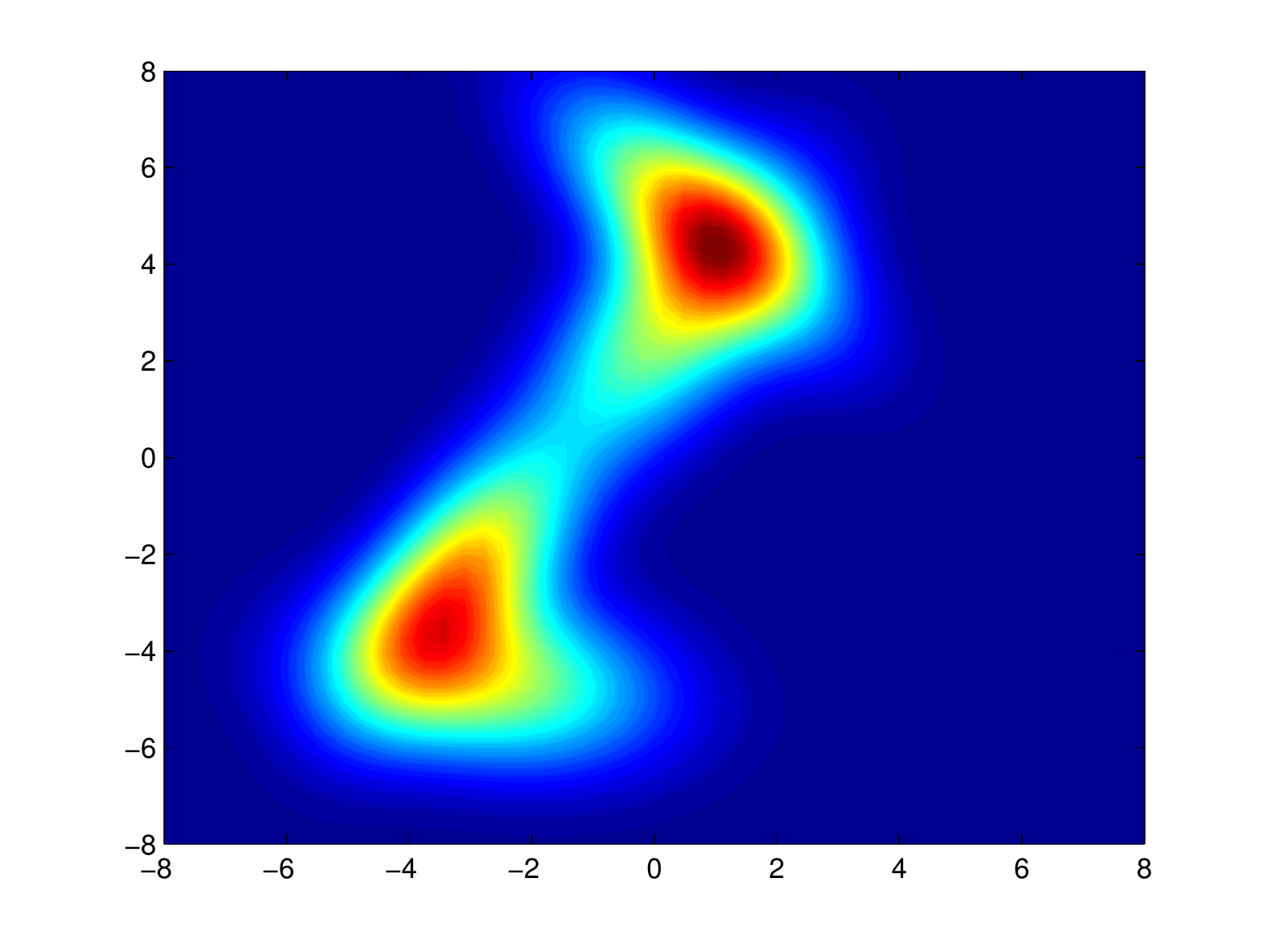}}}
        \subfigure[$k=4$, $n=100$ Points]{\raisebox{0mm}{\includegraphics[width=.235\textwidth]{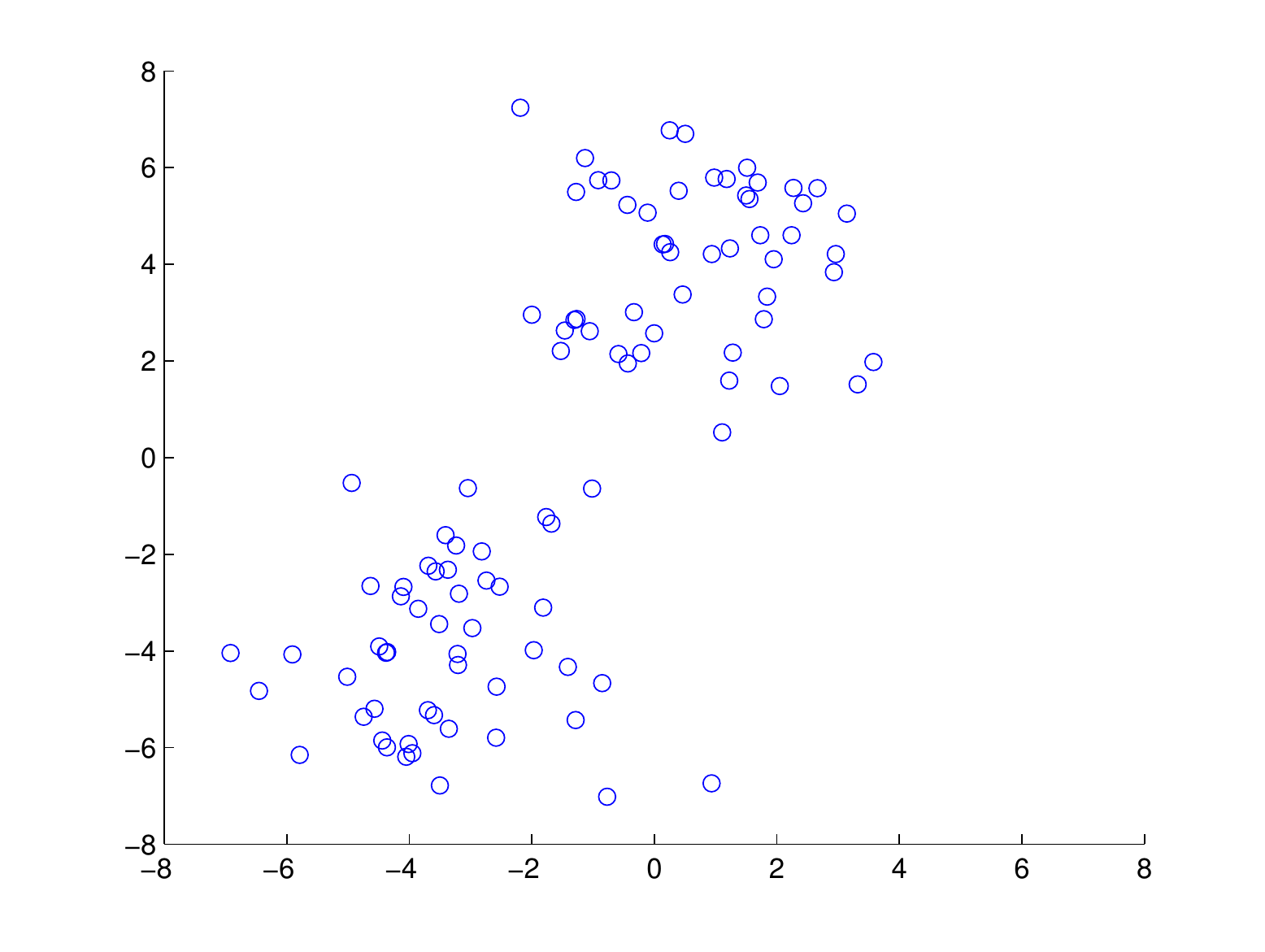}}}
        \subfigure[$k=5$ True Density]{\raisebox{1mm}{\includegraphics[width=.23\textwidth]{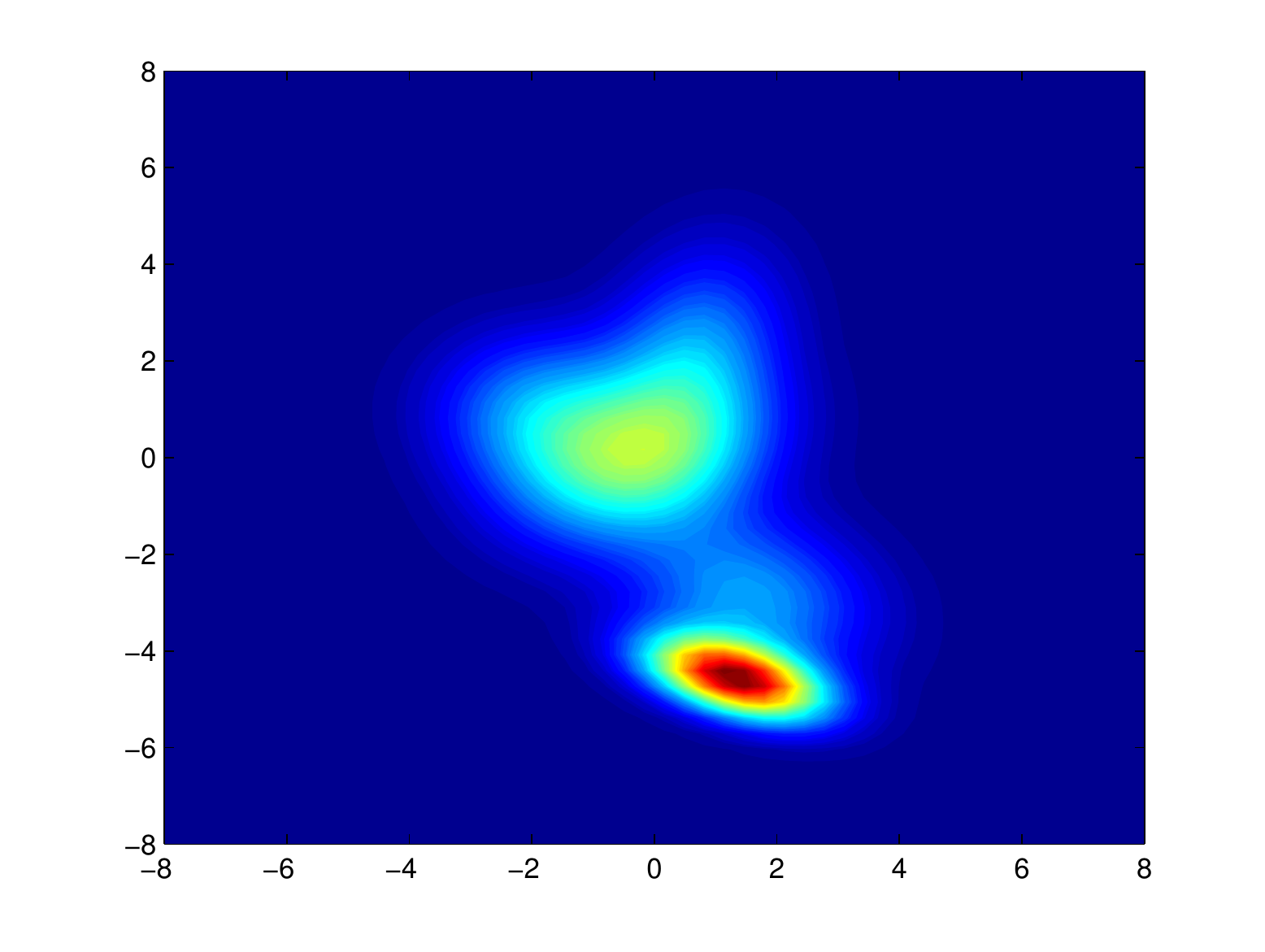}}}
        \subfigure[$k=5$, $n=100$ Points]{\raisebox{0mm}{\includegraphics[width=.235\textwidth]{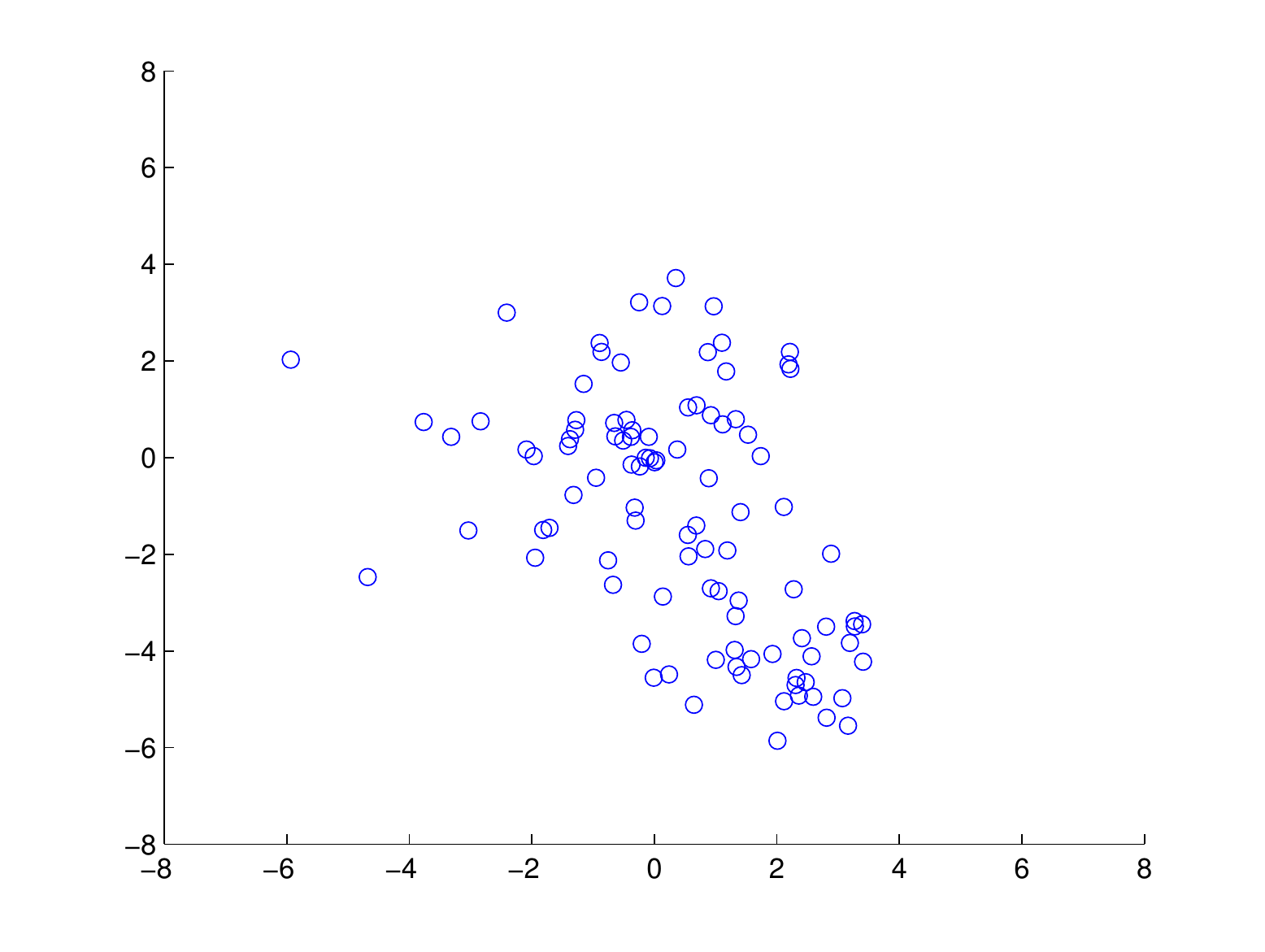}}}
        \vspace{-0.3cm}
        \caption{Typical GMMs generated in our datasets as well as their corresponding samples. One can see that it would be hard for even a human to predict the true number of components, yet the Double-Basis estimator does a good job. \label{fig:typ-GMM} }
\end{figure}

\subsubsection*{Density Estimation Details}
Let $M_t = \{\alpha\ :\ \kappa_\alpha(\nu,\gamma) \leq t \} = \{\alpha_1,\ldots,\alpha_S\}$. First note that: 
\begin{align}
&\EE{\norm{p_i-\tp_i}_2^2} \nonumber \\
=& \EE{\Norm{ \sum_{\alpha\in\Z} a_{\alpha}(P_i)\varphi_{\alpha} -\sum_{\alpha\in M_t} a_{\alpha}(\hP_i) \varphi_{\alpha}}_2^2} \nonumber \\
=& \E\Bigg[\int_{\Lambda^l} \Bigg(\sum_{\alpha\in M_t} (a_{\alpha}(P_i)-a_{\alpha}(\hP_i))\varphi_{\alpha}(x) \nonumber \\
&\quad\quad\quad\quad+\sum_{\alpha\in M_t^c} a_{\alpha}(P_i) \varphi_{\alpha}(x)\Bigg)^2 \ud x\Bigg] \nonumber \\
=& \E \Bigg[\int_{\Lambda^l} \sum_{\alpha\in M_t} \sum_{\rho\in M_t} (a_{\alpha}(P_i)-a_{\alpha}(\hP_i))(a_{\rho}(P_i)-a_{\rho}(\hP_i)) \nonumber\\
&\qquad\qquad\qquad\qquad\varphi_{\alpha}(x)\varphi_{\rho}(x) \ud x\Bigg] \nonumber\\
&+ 2\E \Bigg[\int_{\Lambda^l} \sum_{\alpha\in M_t} \sum_{\rho\in M_t^c} (a_{\alpha}(P_i)-a_{\alpha}(\hP_i))a_{\rho}(P_i) \nonumber\\
&\qquad\qquad\qquad\qquad\qquad\varphi_{\alpha}(x)\varphi_{\rho}(x) \ud x\Bigg]\nonumber\\
&+ \EE{\int_{\Lambda^l} \sum_{\alpha\in M_t^c} \sum_{\rho\in M_t^c} a_{\alpha}(P_i)a_{\rho}(P_i)\varphi_{\alpha}(x)\varphi_{\rho}(x) \ud x} \nonumber \\
=& \EE{\sum_{\alpha\in M_t} (a_{\alpha}(P_i)-a_{\alpha}(\hP_i))^2 }+\EE{\sum_{\alpha\in M_t^c} a_{\alpha}^2(P_i) } \label{eq:derisk},
\end{align}
where the last line follows from the orthonormality of $\{\varphi \}_{\alpha\in\Z}$.
Furthermore, note that $\forall P_i\in \calI$:
\begin{align}
\sum_{\alpha\in M_t^c} a_{\alpha}^2(P_i) =& \frac{1}{t^2} \sum_{\alpha\in M_t^c} t^2a_{\alpha}^2(P_i) \nonumber\\
\leq& \frac{1}{t^2} \sum_{\alpha\in\Z} \kappa^2_\alpha(\nu,\gamma)a_{\alpha}^2(P_i) \nonumber\\
\leq& \frac{\bA}{t^2} \label{eq:turn_err}.
\end{align}
Also,
\begin{align*}
\EE{(a_{\alpha}(P_i)-a_{\alpha}(\hP_i))^2 } =& \left(\EE{a_{\alpha}(\hP_i)}-a_{\alpha}(P_i)\right)^2 \\
&+ \Var\left[a_{\alpha}(\hP_i)\right].
\end{align*}
Clearly, $a_{\alpha}(\hP_i)$ is unbiased from \eqref{eq:coef-hat}. Also,
\begin{align*}
\Var\left[a_{\alpha}(\hP_i)\right] =& \frac{1}{n_i^2}\sum_{j=1}^{n_i} \Var\left[\varphi_\alpha(X_{ij})\right]\\
\leq& \frac{n_i\mxphi^2}{n_i^2}\\
=&O(n_i^{-1}),
\end{align*}
where $\mxphi \equiv \max_{\alpha\in\Z^l}\norm{\varphi_\alpha}_\infty$. Thus,
\begin{align*}
\EE{\norm{p_i-\tp_i}_2^2} \leq \frac{C_1|M_t|}{n_i} + \frac{C_2}{t^2}.
\end{align*}
First note that if we have a bound $\forall \alpha\in M_t,\ |\alpha_i|\leq c_i$ then $|M_t|\leq \prod_{i=1}^l(2c_i+1)$, by a simple counting argument. Let $\lambda=\mathrm{argmin}_i\nu_i^{2\gamma_i}$. For $\alpha\in M_t$ we have: 
\begin{align*}
\sum_{i=1}^l|\alpha_i|^{2\gamma_i} \leq \frac{1}{\nu_\lambda^{2\gamma_\lambda}}\sum_{i=1}^l(\nu_i|\alpha_i|)^{2\gamma_i} = \frac{\kappa^2_\alpha(\nu,\gamma)}{\nu_\lambda^{2\gamma_\lambda}} \leq \frac{t^2}{\nu_\lambda^{2\gamma_\lambda}},
\end{align*} 
and
\begin{align*}
|\alpha_i|^{2\gamma_i} \leq \sum_{i=1}^l|\alpha_i|^{2\gamma_i} \leq {t^2}{\nu_\lambda^{-2\gamma_\lambda}} \implies |\alpha_i| \leq \nu_\lambda^{-\frac{\gamma_\lambda}{\gamma_i}}t^{\frac{1}{\gamma_i}}.
\end{align*}
Thus, $|M_t|\leq \prod_{i=1}^{l}(2\nu_\lambda^{-\frac{\gamma_\lambda}{\gamma_i}}t^{\frac{1}{\gamma_i}}+1)$. Thus, $|M_t|= O\left( t^{\gamma^{-1}} \right)$ where $\gamma^{-1}=\sum_{j=1}^{l}\gamma_j^{-1}$.
Hence,
\begin{align*}
&\frac{\partial}{\partial t} \left[\frac{C_1t^{\gamma^{-1}}}{n_i} + \frac{C_2}{t^2}\right] = \frac{C_1't^{\gamma^{-1}-1}}{n_i} - C_2't^{-3}= 0 & &\implies \\
&t = C n^{\frac{1}{2+\gamma^{-1}}} & &\implies\\
&\EE{\norm{p_i-\tp_i}_2^2} \leq \frac{C_1|M_t|}{n_i} + \frac{C_2}{t^2} =O\left(n_i^{-\frac{2}{2+\gamma^{-1}}}\right).
\end{align*}
Furthermore, by \eqref{eq:derisk} we may see that for $G_i \in \calI$, if 
\begin{align*}
\bar{g}_i = \sum_{\alpha\in\Z} a_{\alpha}(G_i)\varphi_{\alpha},
\end{align*}
then
\begin{align*}
\EE{\norm{g_i-\bar{g}_i}_2^2} = O\left(n_i^{-\frac{2}{2+\gamma^{-1}}}\right).
\end{align*}

\end{document}